\documentclass[preprint]{elsarticle} 

\usepackage[lined,linesnumbered,noend]{algorithm2e}
\usepackage{amssymb}
\usepackage{mathtools}
\usepackage{amsmath}
\usepackage{amsthm}

\usepackage{bm}
\usepackage{esvect}

\newcommand{\vc}[1]{\vec{#1}}
\newcommand{\csigma}{\Sigma_2^P}
\newcommand{\cpi}{\Pi_2^P}

\usepackage{lineno}

\journal{Artificial Intelligence Journal}

\begin{document}

\begin{frontmatter}

\title{Complexity Bounds for the Controllability of Temporal Networks with Conditions, Disjunctions, and Uncertainty}
\author[mit]{Nikhil Bhargava}
\ead{nkb@mit.edu}
\author[mit]{Brian Williams}
\ead{williams@mit.edu}

\address[mit]{Massachusetts Institute of Technology\\
Room 32-227\\
32 Vassar St.\\
Cambridge, MA 02139 USA}


\newtheorem{theorem}{Theorem}
\newtheorem{corollary}{Corollary}[theorem]
\newtheorem{lemma}[theorem]{Lemma}

\theoremstyle{definition}
\newtheorem{defn}{Definition}

\theoremstyle{definition}
\newtheorem{example}{Example}

\begin{abstract}

In temporal planning, many different temporal network formalisms are used to model real world situations. Each of these formalisms has different features which affect how easy it is to determine whether the underlying network of temporal constraints is consistent. While many of the simpler models have been well-studied from a computational complexity perspective, the algorithms developed for advanced models which combine features have very loose complexity bounds. In this paper, we provide tight completeness bounds for strong, weak, and dynamic controllability checking of temporal networks that have conditions, disjunctions, and temporal uncertainty. Our work exposes some of the subtle differences between these different structures and, remarkably, establishes a guarantee that all of these problems are computable in PSPACE.

\end{abstract}


\end{frontmatter}


\section{Introduction}

In temporal planning, many different temporal formalisms are used to model real world situations. The choice of any particular type of network in modeling a problem has inherent trade-offs. If a temporal model supports more features, it can model a given scenario with higher fidelity. However, the additional features come at the expense of performance; modelers care about constructing schedules for temporal networks, and the presence of additional feature types can dramatically slow the runtime of scheduling algorithms.

The computational complexities of many of the simpler temporal models have been well-studied, but the same cannot be said of more advanced models. Despite this gap, there has been considerable effort put into constructing improved algorithms for these feature-rich temporal networks \cite{cimatti:cdstnu,cimatti:execution-language,combi:cstnu,hunsberger:cstnu-execution,venable:dtnu-weak}.

The main contribution of this paper is in providing significantly improved theoretical complexity bounds for computing the controllability of temporal networks with conditions, disjunctions, and temporal uncertainty. The existing bounds for some of these results have been quite loose with most decision problems not known to be better than EXPTIME and some not known to be better than EXPSPACE. We provide completeness results for the strong, weak, and dynamic controllability decision problems across these networks and remarkably prove that all of these problems can be solved in PSPACE. Our results are summarized in Figure \ref{fig:taxonomy}. We conclude with a discussion of our results, giving practical advice to modelers who are interested in the trade-offs of using different temporal networks and lending insight into the differences between these networks.

\begin{figure*}[!tb]
\centering
\includegraphics[width=1.0\textwidth]{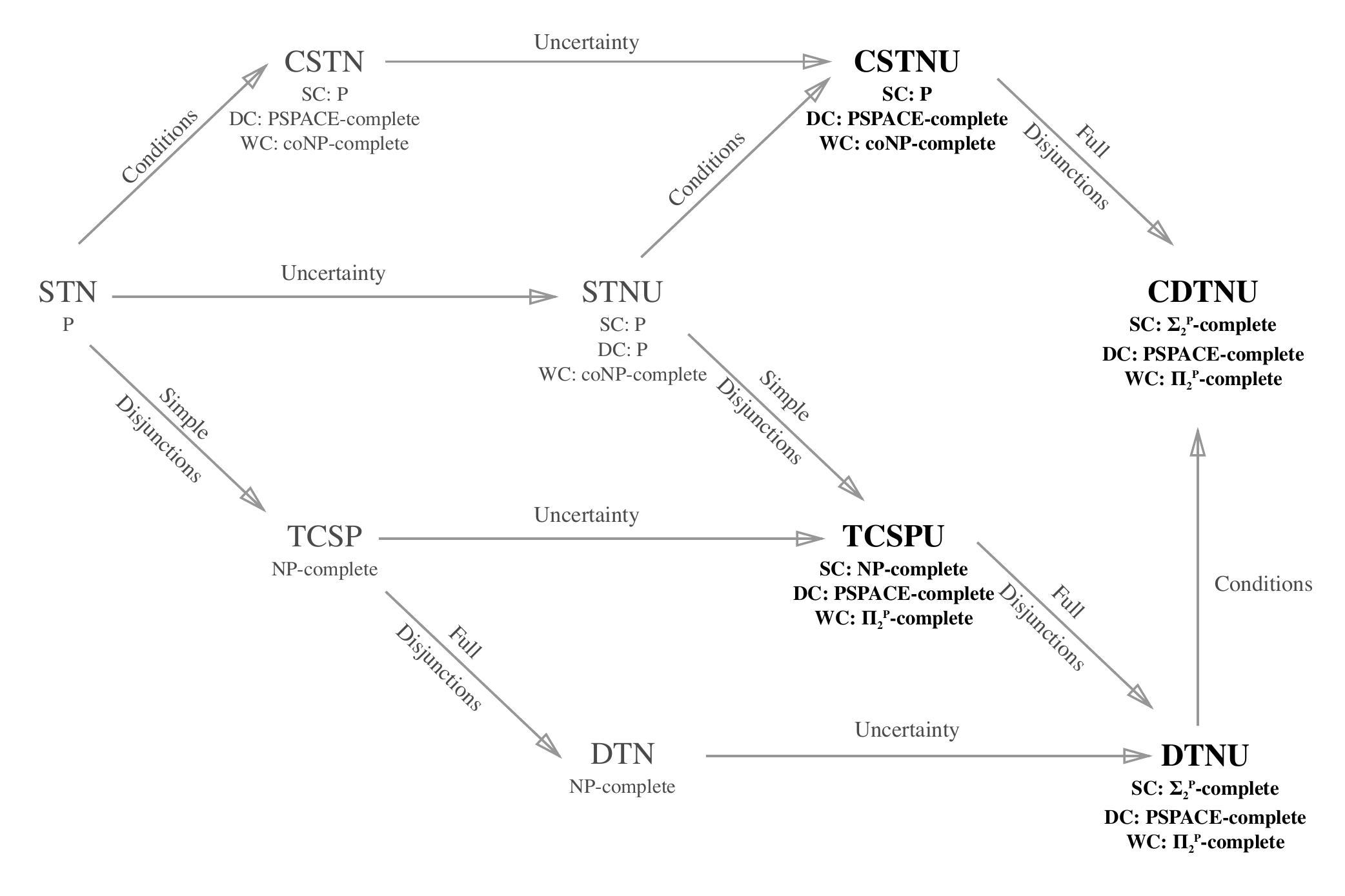}

\caption{A taxonomic organization of temporal networks considered in this paper, how they relate to one another, and the complexity classes to which their decision problems belong. SC, DC, and WC represent strong controllability, dynamic controllability, and weak controllability, respectively. Results in bold represent novel results proved in this paper.}
\label{fig:taxonomy}
\end{figure*}

There are many types of temporal networks beyond those that we focus on in this paper. Many include features related to actor decisions, such as Temporal Plan Networks \cite{kim:tpn}, Temporal Plan Networks with Uncertainty \cite{karpas:tpnu}, Controllable Conditional Temporal Problems \cite{yu:cctp}, Conditional Simple Temporal Networks with Decisions \cite{cairo:cstnd}, and Conditional Simple Temporal Networks with Uncertainty and Decisions \cite{zavatteri:cstnud} while others, such as Probabilistic Simple Temporal Networks \cite{fang:ccpstn} and their relevant extensions, consider probabilistic temporal bounds. Despite the existence of other networks our work covers a broad area of focus that is under active investigation. Future work in this direction will focus on characterizing, organizing, and providing tighter bounds for controllability in these other types of networks but is outside the scope of this work.

\section{Background}

In this section, we will introduce the set of temporal networks whose controllability we will analyze in depth, as well as their simpler base counterparts. The set of networks that we focus on is summarized in Figure \ref{fig:taxonomy}. We divide the discussion of temporal networks into that of base temporal networks, which build on the simplest temporal network representations, and compositional temporal networks, which make use of two or more features in their representation. After describing the temporal networks in detail, we will introduce the complexity classes that make up the polynomial-time hierarchy, as they will be useful in categorizing the complexity of particular controllability classes.

\subsection{Base Temporal Networks}

\subsubsection{Simple Networks}

Simple Temporal Networks (STNs) are the most basic temporal network on which all other formalisms are built \cite{dechter:temporal}. STNs are composed of a set of variables and a set of binary constraints limiting the difference between any two variables (e.g. $B - A \in [10, 20]$). These variables denote individual points in time (henceforth \textit{timepoints}) and the constraints between them are binary temporal constraints limiting their temporal difference (e.g. event $A$ must happen between 10 and 20 minutes before event $B$).

\begin{figure*}[!tb]
\centering
\includegraphics[width=1.0\textwidth]{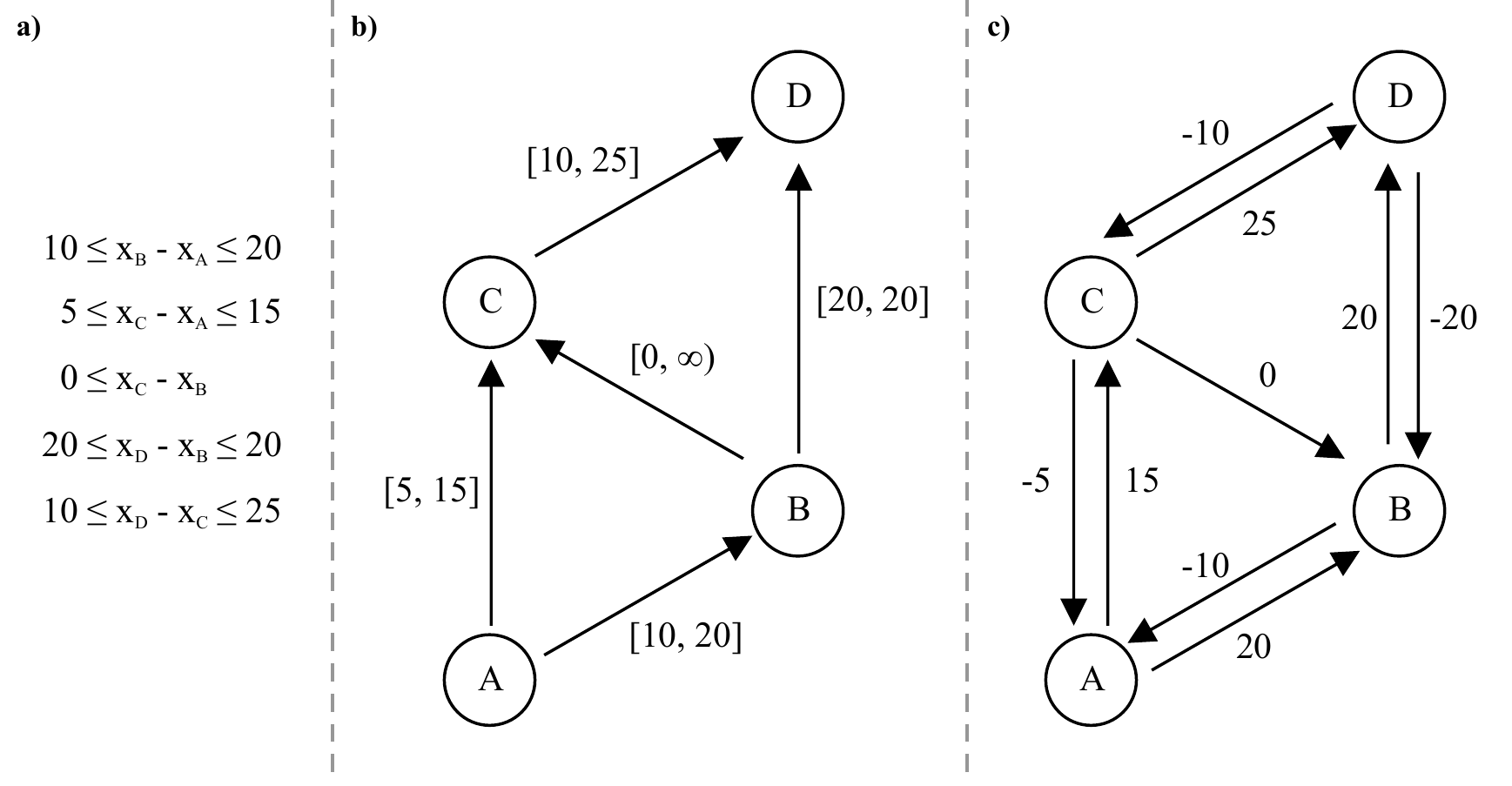}

\caption{(a) An STN as specified by its set of constraints. (b) The same STN represented graphically. (c) The same STN represented using its distance graph formulation.}
\label{fig:stn}
\end{figure*}

\begin{defn}
\textbf{STN \cite{dechter:temporal}}\\
An STN is a 2-tuple $\langle X, R \rangle$ where:
\begin{itemize}
    \item $X$ is a set of timepoint variables composing the temporal network
    \item $R$ is a set of constraints of the form $l_r \leq x_r - y_r \leq u_r$, where $x_r, y_r \in X$ and $l_r, u_r \in \mathbb{R}$
\end{itemize}
\end{defn}

When we consider the feasibility of an STN, we are generally concerned with whether it is possible to construct a \textit{schedule}, or an assignment of values from $\mathbb{R}$ to each variable $x \in X$, such that all constraints are satisfied.

Typically, temporal networks are represented graphically where each timepoint is represented as a node and edges between nodes represent the range of possible differences in value between the endpoints as specified by the original constraints (see Figure \ref{fig:stn}b). We can similarly represent temporal networks by their \textit{distance graphs} \cite{dechter:temporal} in which case we get a simple graph in which each edge from A to B with weight $w$ represents the constraint $x_B - x_A \leq w$, which can be extracted from the original constraints of the temporal network (see Figure \ref{fig:stn}c). In its distance graph formulation, determining the feasibility of an STN reduces to finding a negative cycle in the graph. In general, we take $n$ to be the number of timepoints in a temporal networks and $m$ to be the number of constraints. In the case of STNs, determining feasibility takes at most $O(mn)$ time \cite{dechter:temporal}. When describing modifications made by other temporal networks, we will discuss how these changes affect the cost of checking schedule feasibility.

\subsubsection{Disjunctive Networks}

The first modification we make to STNs is to allow for disjunctions over temporal constraints. In practice, we frequently construct and consider schedules with disjunctive constraints; during a trip to the beach, we know that we want to eat lunch either 30 minutes before swimming or immediately afterwards -- not at any moment in between.

The two types of disjunctive networks that are used in practice, Temporal Constraint Satisfaction Problems (TCSPs) and Disjunctive Temporal Networks (DTNs), differ in terms of the types of disjunctive constraints that they admit \cite{dechter:temporal,stergiou:dtn}.

\begin{defn}
\textbf{TCSP \cite{dechter:temporal}}\\
An TCSP is a 2-tuple $\langle X, R \rangle$ where:
\begin{itemize}
    \item $X$ is a set of timepoint variables composing the temporal network
    \item $R$ is a set of simple disjunctive constraints of the form $x_r - y_r \in \bigcup\limits_k [l_{r, k}, u_{r,k}]$, where $x_r, y_r \in X$ and $l_{r,k}, u_{r, k} \in \mathbb{R}$
\end{itemize}
\end{defn}

\begin{defn}
\textbf{DTN \cite{stergiou:dtn}}\\
An DTN is a 2-tuple $\langle X, R \rangle$ where:
\begin{itemize}
    \item $X$ is a set of timepoint variables composing the temporal network
    \item $R$ is a set of full disjunctive constraints of the form $\bigvee\limits_k \left(l_{r,k} \leq x_{r,k} - y_{r,k} \leq u_{r,k}\right)$, where $x_{r,k}, y_{r,k} \in X$ and $l_{r,k}, u_{r,k} \in \mathbb{R}$
\end{itemize}
\end{defn}

The disjunctive constraints of TCSPs require that every constraint in a given disjunction relates the same pair of timepoints. In contrast, DTNs allow disjunctive constraints to be a disjunction over any constraints that might be found in an STN. In this paper, we will refer to the type of disjunctions allowed by TCSPs as \textit{simple disjunctions} and the type of disjunctions allowed by DTNs as \textit{full disjunctions}. Checking the feasibility of both TCSPs and DTNs is known to be NP-complete \cite{dechter:temporal,stergiou:dtn}. It is worth noting that a linear time transformation exists that converts DTNs into equivalent TCSPs \cite{planken:tcsp-to-dtn}, but maintaining the distinction between the two is important because, remarkably, as we extend the two types of networks, we see that the computational complexity of solving them begins to diverge.

\subsubsection{Conditional Networks}

The Conditional Simple Temporal Network (CSTN) represents a different way in which we can augment STNs \cite{tsamardinos:cstn}. CSTNs allow for the introduction and observation of uncontrollable events and the conditional enforcement of constraints based on the observations of those events.

\begin{defn}
\textbf{CSTN \cite{tsamardinos:cstn}}\\
A CSTN is a tuple $\langle X, R, P, O \rangle$ where:
\begin{itemize}
    \item $X$ is a set of timepoint variables composing the temporal network
    \item $R$ is a set of constraints of the form $\psi_r \rightarrow \left(l_r \leq x_r - y_r \leq u_r \right)$, where $x_r, y_r \in X$, $\psi_r$ is a label representing a conjunction of propositions or their negations, and $l_r, u_r \in \mathbb{R}$
    \item $P$ is a set of propositions
    \item $O$ is a function mapping propositions in $P$ to the timepoints where their values are observed
\end{itemize}
\end{defn}

To illustrate the usefulness of CSTNs, we provide an example. If we want to schedule the delivery of a package, we may prefer to use a CSTN to encode the urgency of the request; a package that we see marked as urgent, may need to be scheduled in the next 24 hours, but a package that is not marked as such can use a more relaxed schedule that guarantees shipment within the next week. If we have a timepoint $A$ representing when the package goes out for delivery and $B$ as the timepoint representing when the package must be delivered, we can encode the urgency using two constraints, if the package is urgent, we have the constraint $B - A \leq 1d$ with label $u$, and if the package is not urgent, we have the constraint $B - A \leq 7d$ with label $\neg u$. What makes scheduling over CSTNs notable is that we may learn about the value of proposition $u$, or in this case the urgency of the package, at some unrelated timepoint $C$ that may differ from the timepoints associated with the constraints they affect. In our given example, $C$ represents the time at which the customer tells us the package's urgency. It is possible that the customer indicates that the package is urgent the day before dropping it off, but it is equally possible that the customer tells us the package is urgent several hours after they have already dropped it off. We conditionally enforce labeled constraints by observing the realized values of the propositions and checking whether a constraint's label, $\psi_r$ is true. In the package example, we know that we will only need to enforce one of the two constraints based on what the observed value of $u$ is at timepoint $C$. We use the function $O$ to encode the timepoints at which specific propositions are observed.

Importantly, the true values of propositions are not ``scheduled'' in the same way that timepoints are. Different instantiations of the same problem may yield different values for the propositions and, correspondingly, result in different constraints that must be enforced during execution. As a result, the scheduling problem for CSTNs is different than the one for STNs, TCSPs, and DTNs. In the previously described temporal networks, we knew the full set of constraints that would be enforced prior to scheduling and as such could satisfy all constraints with an implicitly static schedule. However, with CSTNs, there is no predetermined guarantee about when the scheduler learns about propositions, as the scheduler may have to predetermine a schedule that is robust to any learned proposition values or may have the flexibility to adapt the schedule on the fly. Across these different situations, different decisions may be made with respect to scheduling that may trade off between learning the actual values of propositions early in execution and maintaining a buffer of temporal flexibility. As such, when checking feasibility of CSTNs, we use \textit{strong}, \textit{weak}, and \textit{dynamic consistency} to denote the different models under which the scheduler is guaranteed to learn the actual proposition values \cite{tsamardinos:cstn}.

Strong consistency implies there exists a schedule that can be constructed that assigns values to all timepoints in $X$, such that for every realization of the values of the propositions in $P$, all constraints in $R$ are satisfied. Strong consistency checking of a CSTN reduces to checking the temporal consistency of the underlying STN and so is computable in $O(mn)$ time \cite{tsamardinos:cstn}. A CSTN is weakly consistent if for every assignment of values to the propositions in $P$, there exists some schedule can be constructed assigning values to timepoint variables in $X$, such that all constraints in $R$ are satisfied. Weak consistency checking of CSTNs is coNP-complete \cite{tsamardinos:cstn}. Dynamic consistency is concerned with whether it is possible to dynamically construct a schedule where assignment to values in $X$ happen in order of timepoint values and the true values of propositions $p \in P$ are learned only when the corresponding timepoint given by $O(p)$ is executed. Dynamic consistency checking in CSTNs is PSPACE-complete \cite{cairo:cstn-pspace-complete}.

\subsubsection{Networks with Temporal Uncertainty}

All the formalisms we have introduced to this point have one major shortcoming; namely, they assume that the scheduler has absolute control over each and every timepoint. In practice, there exist few scenarios in which an agent has that type of total control. Agents are unable to control how much traffic will affect their morning commute or when it might start to rain. To account for this, we need a way to augment temporal networks to capture the difficulty of planning around uncertain events.

Simple Temporal Networks with Uncertainty (STNUs) extend STNs, allowing us to model events whose timings are outside the control of the scheduler but still closely related to the actions taken by the scheduler \cite{vidal:controllability}. By convention, we say that those events and actions not explicitly chosen by the scheduler are chosen by nature.

\begin{defn}
\textbf{STNU \cite{vidal:controllability}}\\
An STNU is a 4-tuple $\langle X_e, X_c, R_r, R_c \rangle$ where:
\begin{itemize}
    \item $X_e$ is the set of executable timepoint variables
    \item $X_c$ is the set of contingent timepoint variables
    \item $R_r$ is the set of requirement constraints of the form $l_r \leq x_r - y_r \leq u_r$, where $x_r, y_r \in X_c \cup X_e$ and $l_r, u_r \in \mathbb{R}$
    \item $R_c$ is the set of contingent constraints of the form $0 \leq l_r \leq c_r - e_r \leq u_r$, where $c_r \in X_c$, $e_r \in X_e$ and $l_r, u_r \in \mathbb{R}$
\end{itemize}
\end{defn}

In STNUs, timepoints are subdivided into executable and contingent timepoints and constraints are subdivided into requirement and contingent ones. Executable timepoints are the timepoints that the scheduler is responsible for, whereas contingent timepoints are scheduled by nature. Requirements constraints behave like ordinary STN constraints and are free to constrain any pair of timepoints. Contingent constraints, in contrast, represent relations between a starting executable timepoint and an ending contingent timepoint that nature is guaranteed to enforce. By convention, the lower bound of a contingent constraint is required to be non-negative to enforce that the starting timepoint of the constraint, in some sense, causes the ending timepoint. It is worth noting that we require that all contingent constraints begin from an executable timepoint. This does not have an impact on the expressivity of our networks, as it is simple to take a pair of chained contingent links and interrupt them with a simple zero-length requirement constraint, but it will simplify our proof exposition.

For simplicity, we may sometimes refer to a contingent action duration. For a contingent constraint $r$, represented as $l_r \leq c_r - e_r \leq u_r$, we say that the \textit{duration} of $r$ is the value specified by the difference $c_r - e_r$. When the contingent timepoints $x_e$ are picked by nature, it is equivalent to nature picking a series of durations for contingent actions, as the set of contingent action durations together with the set of executable timepoints uniquely determines a set of contingent timepoints.

As was the case with CSTNs, when modeling the execution of STNUs, there are events that are outside of the control of the scheduler that force us to consider schedulability in the context of how we eventually learn about nature's actions. While CSTN feasibility was centered around strong, weak, and dynamic consistency, STNU feasibility is based on \textit{strong}, \textit{weak}, and \textit{dynamic controllability} \cite{vidal:controllability}.

We say that an STNU is strongly controllable if there exists a schedule for all executable timepoints $X_e$, such that for every possible assignment of values to contingent timepoints in $X_c$ that satisfy the contingent constraints $R_c$, all of the requirement constraints $R_r$ are satisfied. STNU strong controllability checking, much like STN consistency checking, reduces to detecting the presence of a negative cycle and can be computed in $O(mn)$ time \cite{vidal:controllability}.

Weak controllability asks whether it is possible to reactively construct a schedule if the durations of the uncertain events are revealed before scheduling begins. In other words, for every fully specified set of contingent action durations that guarantee satisfaction of contingent constraints $R_c$, weak controllability asks whether it is always possible to pick a set of executable timepoints $X_e$ such that all requirement constraints $R_r$ are satisfied. While checking whether a schedule exists for any one particular realization of the uncertain events reduces to checking STN consistency, checking STNU weak controllability in general is coNP-complete \cite{morris:waypoint}.

Dynamic controllability considers the question of whether it is possible to create a just-in-time schedule, where timepoints $X_e$ and $X_c$ are chosen interactively in order of their values (with the guarantee that all $X_c$ are chosen such that all contingent constraints $R_c$ are satisfied), that guarantees the satisfaction of all requirement constraints $R_r$. Dynamic controllability of an STNU can be computed in polynomial time and more recently was shown to have a worst-case $O(n^3)$ runtime \cite{morris:n5,morris:n3}.

We will use these same notions of strong, weak, and dynamic controllability when we extend the STNU to include conditions and disjunctions.

\subsection{Compositional Temporal Networks}

The compositional temporal networks that we choose to focus on in this paper will be the ones that augment STNUs with disjunctions and conditional constraints.

\subsubsection{Disjunctions and Temporal Uncertainty}

We start by adding disjunctions to STNUs. As was the case with disjunctions added to STNs, when considering disjunctive temporal networks with uncertainty, we consider the effects of allowing both simple and full disjunctions.

Temporal Constraint Satisfaction Problems with Uncertainty (TCSPUs) augment STNUs by adding simple disjunctions over constraints.

\begin{defn}
\textbf{TCSPU \cite{venable:dtnu-weak}}\\
A TCSPU is a 4-tuple $\langle X_e, X_c, R_r, R_c \rangle$ where:
\begin{itemize}
    \item $X_e$ is the set of executable timepoint variables
    \item $X_c$ is the set of contingent timepoint variables
    \item $R_r$ is the set of simple disjunctive temporal constraints over $X_c \cup X_e$
    \item $R_c$ is the set of simple disjunctive contingent constraints
\end{itemize}
\end{defn}

By augmenting a TCSPU with full disjunctions over temporal constraints, we get Disjunctive Temporal Networks with Uncertainty (DTNUs) \cite{venable:dtnu-def}.

\begin{defn}
\textbf{DTNU \cite{peintner:dtnu-strong}}\\
An DTNU is a 4-tuple $\langle X_e, X_c, R_r, R_c \rangle$ where:
\begin{itemize}
    \item $X_e$ is the set of executable timepoint variables
    \item $X_c$ is the set of contingent timepoint variables
    \item $R_r$ is the set of full disjunctive temporal constraints over $X_c \cup X_e$
    \item $R_c$ is the set of simple disjunctive contingent constraints
\end{itemize}
\end{defn}

It is worth noting that for DTNUs, all disjunctive contingent constraints are simple. Most models of temporal uncertainty assume that the duration of a contingent link is independent of any action taken by the scheduler. Accordingly, allowing disjunctive constraints to span different contingent links or to span contingent and requirement links would violate the spirit of this approach.

The concepts of strong, weak, and dynamic controllability as defined for STNUs scale immediately to temporal networks with disjunctions. However, the introduction of disjunctions makes the act of computing controllability much more difficult. The best available algorithms for deciding strong controllability of temporal networks with uncertainty and disjunction are in EXPSPACE \cite{peintner:dtnu-strong}. Dynamic and weak controllability of these networks can be computed in exponential time, but these approaches also use exponential space. It is unknown whether any form of controllability checking for DTNUs or TCSPUs can be done in polynomial space \cite{cimatti:cdstnu,venable:dtnu-def}.

\subsubsection{Conditions and Temporal Uncertainty}

Extending STNUs with conditional constraints gives us Conditional Simple Temporal Networks with Uncertainty (CSTNUs) \cite{hunsberger:cstnu-def}.

\begin{defn}
\textbf{CSTNU \cite{combi:cstnu}}\\
A CSTNU is a tuple $\langle X_e, X_c, R_e, P, O \rangle$ where:
\begin{itemize}
    \item $X_e$ is a set of executable timepoint variables
    \item $X_c$ is a set of contingent timepoint variables
    \item $R_r$ is a set of requirement constraints of the form $\psi_r \rightarrow \left(l_r \leq x_r - y_r \leq u_r \right)$, where $x_r, y_r \in X_e \cup X_c$, $\psi_r$ is a label representing a conjunction of propositions or their negations, and $l_r, u_r \in \mathbb{R}$
    \item $R_c$ is a set of contingent constraints of the form $0 \leq l_r \leq c_r - e_r \leq u_r$, where $c_r \in X_c, e_r \in X_e$ and $l_r, u_r \in \mathbb{R}$
    \item $P$ is a set of propositions
    \item $O$ is a function mapping propositions in $P$ to the timepoints where their values are observed
\end{itemize}
\end{defn}

With CSTNUs, we now have two sources of external uncertainty, the observed values of propositions and the realized durations of contingent links. While we could evaluate consistency and controllability conditions separately (e.g. checking whether a network is strongly consistent while being dynamically controllable), we typically consider the two jointly. In other words, we assume that both the durations of contingent links and the values of the propositions are either never observed, all observed before execution, or observed along the way when we evaluate strong, weak, and dynamic controllability, respectively. Dynamic controllability of CSTNUs belongs to EXPTIME \cite{cimatti:cdstnu}, but the complexity of checking strong and weak controllability are still open questions.

\subsubsection{Conditions, Disjunctions, and Temporal Uncertainty}

Finally, we combine conditions, disjunctions, and temporal uncertainty in a single network to get Conditional Disjunctive Temporal Networks with Uncertainty (CDTNUs).

\begin{defn}
\textbf{CDTNU}\\
A CDTNU is a tuple $\langle X_e, X_c, R_e, P, O \rangle$ where:
\begin{itemize}
    \item $X_e$ is a set of executable timepoint variables
    \item $X_c$ is a set of contingent timepoint variables
    \item $R_e$ is a set of requirement constraints of the form\\$\bigvee\limits_k \psi_{r,k} \rightarrow \left(l_{r,k} \leq x_{r,k} - y_{r,k} \leq u_{r,k} \right)$, where $x_{r,k}, y_{r,k} \in X$, $\psi_{r,k}$ is a label representing a conjunction of propositions or their negations, and $l_{r,k}, u_{r,k} \in \mathbb{R}$
    \item $R_c$ is a set of simple disjunctive contingent constraints
    \item $P$ is a set of propositions
    \item $O$ is a function mapping propositions in $P$ to the timepoints where their values are observed
\end{itemize}
\end{defn}

We can apply the same techniques as those found in CSTNUs and DTNUs to show that dynamic controllability of CDTNUs can be computed in EXPTIME \cite{cimatti:cdstnu}. Algorithms for strong and weak controllability of CDTNUs have not yet been developed.

\subsection{Polynomial Time Hierarchy}

Before we continue to the actual complexity results it is useful to introduce the polynomial-time hierarchy \cite{stockmeyer:polynomial-hierarchy}, as it will allow us to more precisely characterize the difficulty of some of our controllability problems.

The classes $\Sigma_k^P$ and $\Pi_k^P$ are defined recursively. We start with $\Sigma_1^P =$ NP and $\Pi_1^P =$ coNP and define $\Sigma_{k+1}^P$ as NP$^{\Sigma_k^P}$ and $\Pi_{k+1}^P$ as coNP$^{\Sigma_k^P}$, where $A^B$ represents the set of problems that can be solved in complexity class $A$ if an oracle for a $B$-complete problem is provided.

In this paper, we will pay close attention to the complexity classes $\csigma$ and $\cpi$ and will make heavy use of the fact that $\Sigma_{k}^P = $ co$\Pi_k^P$ and that $\forall\exists$3SAT is a $\cpi$-complete problem, where $\forall\exists$3SAT is the problem of determining whether for a given 3-CNF $\Phi(\vc{x}, \vc{y})$ it is the case that for all $\vc{y}$, there exists $\vc{x}$, such that $\Phi(\vc{x}, \vc{y})$ is true \cite{stockmeyer:polynomial-hierarchy}. $\Sigma_k^P$ and $\Pi_k^P$ are also known to be fully contained within PSPACE, meaning that membership to any complexity class in the polynomial-time hierarchy guarantees the existence of a deterministic algorithm that uses at most polynomial space.

\section{Complexity}

While complexity results for the base temporal networks we have described are well-known, very few tight bounds exist for the networks derived from their composition, despite the fact that much work has been done to develop algorithms for them. Many of their hardness lower-bounds can be inherited from the base temporal networks, but it is an open question whether or not they are tight.

In this section, we will prove complexity class completeness results for each of strong, weak, and dynamic controllability for each network, updating the hardness lower-bounds as needed before demonstrating membership to the appropriate class. When describing the controllability decision problems, we will use the prefixes SC-, WC-, and DC- to refer to checking the strong, weak, and dynamic controllability of the denoted temporal network, respectively.

\subsection{Hardness Results}

We start by providing tighter hardness lower-bounds for the controllability problems across temporal networks. Existing results for CSTNs give us appropriate lower-bounds for CSTNUs, but for the temporal networks with disjunction and uncertainty, we need tighter analysis than the NP-hardness provided by TCSPs and DTNs.


\begin{lemma}
\label{wc-tcspu-hard}
Checking the weak controllability of a TCSPU is $\cpi$-hard.
\end{lemma}

\begin{proof}
To show WC-TCSPU is $\cpi$-hard, we will provide a reduction from $\forall\exists$3SAT. In other words, we want to construct a TCSPU $T$ such that a formula $\forall \vc{y}, \exists \vc{x} : \phi(\vc{x}, \vc{y})$ is weakly controllable if and only if $T$ is weakly controllable, where $\vc{x}, \vc{y}$ are vectors of boolean values, and $\phi$ is a 3-CNF formula.

We start by defining our timepoints, starting with a reference timepoint $Z$. For each $x_i$, we construct timepoint $t_{x_i}$ with disjunctive constraint $t_{x_i} - Z \in [0, 0] \cup [1, 1]$. For each $y_j$, we also construct timepoint $t_{y_j}$ with contingent constraint $t_{y_j} - Z \in [0, 0] \cup [1, 1]$. These timepoints will represent the initial values chosen against which we will evaluate $\phi$ with 0 corresponding to an assignment of false and 1 corresponding to true.

For convenience, we also add timepoints corresponding to the negations of each variable. $t_{\overline{x}_i}$ has two corresponding constraints, $t_{\overline{x}_i} - Z \in [0, 0] \cup [1, 1]$ and $t_{\overline{x}_i} - t_{x_i} \in [-1, -1] \cup [1, 1]$. This ensures that $t_{\overline{x}_i}$ takes on a different value than $t_{x_i}$. Similarly, we add new timepoints $t_{\overline{y}_j}$ with requirement links $t_{\overline{y}_j} - Z \in [0, 0] \cup [1, 1]$ and $t_{\overline{y}_j} - t_{y_j} \in [-1, -1] \cup [1, 1]$. We will rely on the fact that we are evaluating weak controllability to ensure that we set the timepoints for the negated variables in response to the values assigned by nature.

We now move on to encoding each individual clause of $\phi$ into our TCSPU $T$. Our approach is going to be highly inspired by the reduction from 3SAT to the 3-coloring problem on graphs and the reduction from 3-coloring to computing feasibility of a TCSP \cite{dechter:temporal}. We will emulate the three colors by requiring all timepoints to occur at time 0, 1, or 2 and enforce that two nodes $t_i, t_j$ differ in value by requiring that $t_i - t_j \in \{-2, -1, 1, 2\}$.

\begin{figure}[!tb]
\centering
\includegraphics[width=0.75\textwidth]{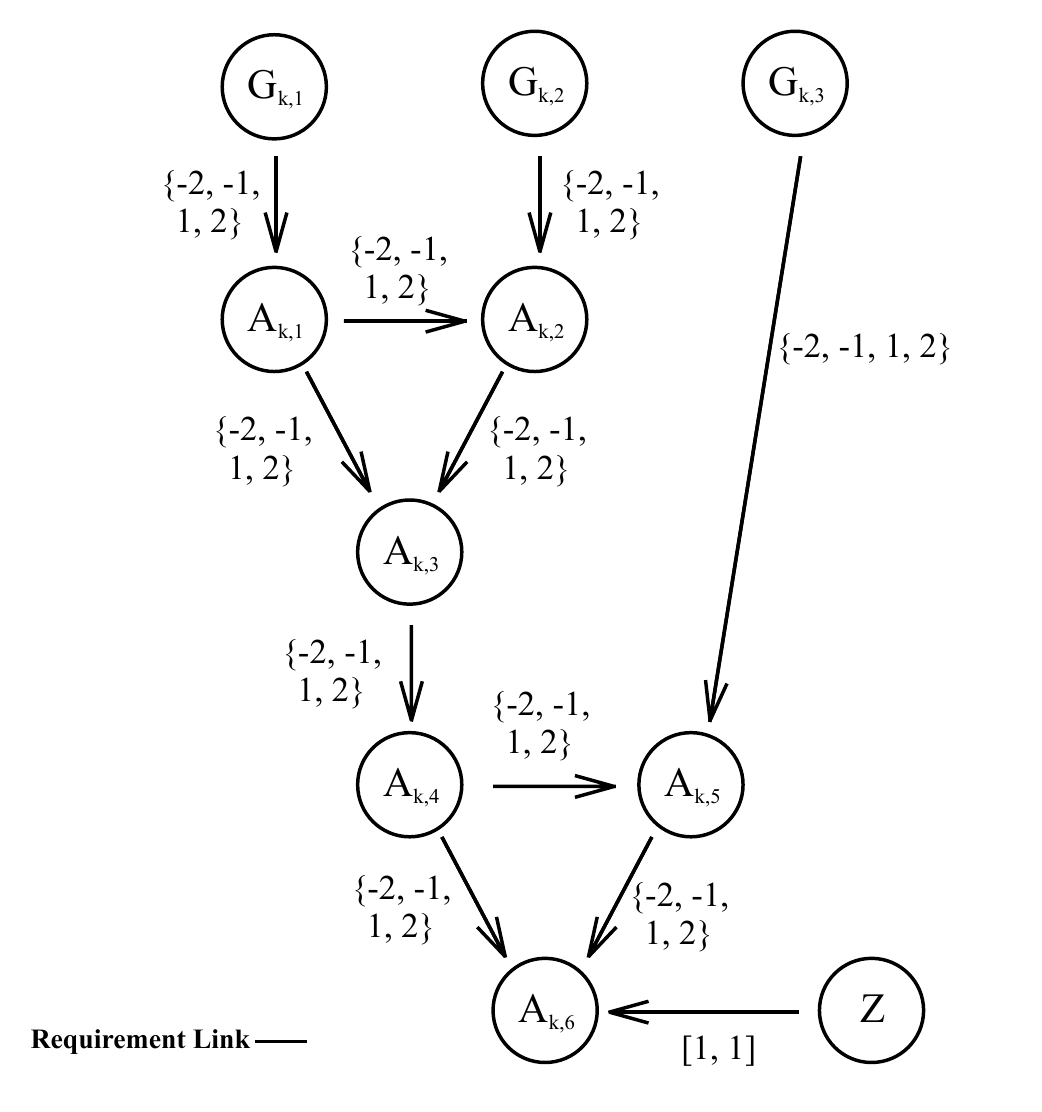}

\caption{A gadget used in the proof that WC-TCSPU is $\cpi$-hard. The $A_k$ timepoints can each take on any value from $\{0, 1, 2\}$. The value $A_{k, 6}$ represents the disjunction of $G_{k,1}, G_{k,2}, G_{k,3}$ and is constrained to equal one.}
\label{fig:wc-tcspu}
\end{figure}

For each clause $c_k$ of $\psi$, we create a new gadget whose output represents the truth value of $c_k$ (see Figure \ref{fig:wc-tcspu}). Each timepoint $G_{k,l}$ represents the truth value of literal $l$ of clause $c_k$. We require that the value matches the initially assigned value of literal $q$ by adding the constraint $G_{k,l} - t_q = 0$. The layout of timepoints $A_k$ weakly emulate an or-gate, where $A_{k,6}$ is the output and constrained to have a value of 1. For any values of the timepoints $G_k$, it is possible to assign all of the timepoints $A_k$ such that $A_{k,6} = 1$ except for when $G_{k,1} = G_{k,2} = G_{k,3} = 0$. As a result, it is possible to choose a set of values for the timepoints to satisfy the constraints of the gadget so long as at least one literal of the original clause $c_k$ is true.

Taken together, if there exists an assignment of values to timepoints such that each gadget's constraints are satisfied, then for whichever particular $\vc{y}$ we start with, then $\exists \vc{x} : \phi(\vc{x}, \vc{y})$. When checking weak controllability, all executable timepoints are assigned values after the contingent timepoints, so as we have constructed it, $T$ is weakly controllable if and only if $\forall \vc{y}, \exists \vc{x} : \phi(\vc{x}, \vc{y})$. Thus, WC-TCSPU is $\cpi$-hard.
\end{proof}

\begin{lemma}
\label{dc-tcspu-hard}
Checking the dynamic controllability of a TCSPU is PSPACE-hard.
\end{lemma}

\begin{proof}
To show that DC-TCSPU is PSPACE-hard, we provide a reduction from TQBF, which is known to be PSPACE-complete, to DC-TCSPU. In particular, for a problem of the form $\exists x_1 \forall y_1 ... \exists x_n \forall y_n : \phi(\vc{x}, \vc{y})$, we can construct a TCSPU $T$ such that $T$ is dynamically controllable if and only if $\exists x_1 \forall y_1 ... \exists x_n \forall y_n : \phi(\vc{x}, \vc{y})$, where $\phi$ is a 3-CNF formula.

Ideally, we would employ a strategy similar to our transformation for WC-TCSPU in Lemma \ref{wc-tcspu-hard}, but in that construction, many of the clausal gadget timepoints can occur before the contingent timepoints they relate to are assigned by nature. Because dynamic controllability requires timepoints to be assigned reactively in a just-in-time manner, we must make sure that all values of $\vc{y}$ are encoded and specified by the network before we do any subsequent computation.

We start by encoding the alternating choice of $x_i$ and $y_i$ as represented by the values decided by the scheduler and nature. We start with an anchor point $O$ and for each $x_i$ and $y_i$, we create timepoints $\tau_{x_i,s}$, $\tau_{x_i,e}$, $\tau_{y_i,s}$, and $\tau_{y_i,e}$. For each $x_i$, we create a requirement constraint of $\tau_{x_i,e} - \tau_{x_i,s} \in [0, 0] \cup [1, 1]$, and for each $y_i$, we create a contingent constraint of $\tau_{y_i,e} - \tau_{y_i,s} \in [0, 0] \cup [1, 1]$. This enforces that the difference between the start and end values is either 0 or 1, corresponding to an assignment of false or true in the original formula. To ensure that the values are chosen in order when evaluated in a dynamic controllability setting, we require that $\tau_{x_i,s} - O = 2i - 2$ and that $\tau_{y_i,s} - O = 2i - 1$. This gives us the exact alternating pattern as described by the original formula, and what remains is to evaluate the truth condition.

Our strategy for evaluating the truth of the formula is to replicate the same structures used by the constructed TCSPU in Lemma \ref{wc-tcspu-hard}. We create a secondary anchor point $Z$ with $Z - O = 2n + 2$ to ensure that $Z$ happens after all boolean values have been assigned, and then create new timepoints corresponding to the values of $\vc{x}$ and $\vc{y}$ that are anchored at $Z$ instead of at different times during the execution. For each $x_i$, we create $t_{x_i}$ with the constraint $t_{x_i} - \tau_{x_i,e} = 2(n - i) + 4$, and for each $y_i$, we create $t_{y_i}$ with the constraint $t_{y_i} - \tau_{y_i,e} = 2(n - i) + 3$. The rest of the construction, namely the construction of the negated literal values and the clausal gadgets, remains the same, and by the same reasoning, we see that it is possible for a given assignment, it is possible for all constraints to be respected if and only if $\phi$ is satisfied by that assignment of values. Since the initial timepoints are set up such that when the entire network is dynamically controllable the values of timepoints are chosen in the same order as the quantification of the original TQBF formula, we know that $T$ is dynamically controllable if and only if $\exists x_1 \forall y_1 ... \exists x_n \forall y_n : \phi(\vc{x}, \vc{y})$. Because the new network can be constructed in polynomial time, we have a polynomial time reduction from TQBF to DC-TCSPU, so DC-TCSPU is PSPACE-hard.

\end{proof}

\begin{lemma}
\label{sc-dtnu-hard}
Checking the strong controllability of a DTNU is $\csigma$-hard.
\end{lemma}

\begin{proof}
To prove that SC-DTNU is $\csigma$-hard, we will reduce the complement of $\forall\exists$3SAT, a $\cpi$-complete problem, to SC-DTNU.

An example problem of $\forall\exists$3SAT is of the form $\forall \vc{x}, \exists \vc{y} : \phi(\vc{x}, \vc{y})$, where $\vc{x}, \vc{y}$ are vectors of boolean values and $\phi$ is a 3-CNF formula. The complementary problem is $\exists \vc{x}, \forall \vc{y} : \psi(\vc{x}, \vc{y})$, where $\psi$ is a 3-DNF formula representing the negation of $\phi$. Given the input problem, we construct a corresponding DTNU $D$ that is strongly controllable if and only if the complementary formula $\psi$ is true (if the original formula $\phi$ is false).

First we define the timepoints of $D$. We start with a reference timepoint $Z$, which represents the first point to be executed. For each $x_i \in \vc{x}$, we add points $t_{x_i}$ and $t_{\overline{x}_i}$ to represent the value of $x_i$ and its negation during some candidate assignment to our formula. We do the same thing for $\vc{y}$ adding $t_{y_j}$ and $t_{\overline{y}_j}$ for each $y_j \in \vc{y}$. We also introduce a new gadget per clause of $\psi$ (see Figures \ref{fig:sc-dtnu-weak-or} and \ref{fig:sc-dtnu-gadget}) and in each gadget, we introduce ten new timepoints. Timepoints $G_{k, 1}$, $G_{k, 2}$, and $G_{k, 3}$ represent the values of each literal of clause $k$ and timepoint $G_{k, and}$ represents the value of the conjunction of those literals. For each clause, we also add $A_{k, 1}$, $A_{k, 2}$, $A_{k, 3}$, $A_{k, 4}$, $A_{k, 5}$, and $A_{k, 6}$ which are used collectively to simulate an and clause. By appropriately adding contingent and requirement links between these timepoints, we will get a DTNU that is controllable if and only if the original formula $\psi$ is true.

We start by adding constraints to encode the initial assignment of values. For each $t_{x_i}$ we add a simple disjunctive constraint requiring that $t_{x_i} - Z \in [0, 0] \cup [1, 1]$. Similarly, for each $t_{y_j}$, we add a disjunctive \textit{contingent} constraint enforcing $t_{y_j} - Z \in [0, 0] \cup [1, 1]$. The choice of values for these initial timepoints maps directly back to an assignment of values in the 3-DNF formula $\psi$ with 0 representing false and 1 representing true.

We also enforce the values of the negations of these variables for convenience, with the same simple disjunctive constraint requiring $t_{\overline{x}_i} - Z \in [0, 0] \cup [1, 1]$ and the disjunctive contingent constraint enforcing $t_{\overline{y}_j} - Z \in [0, 0] \cup [1, 1]$. To ensure that $x_i$ and its negation take on values we also add the requirement that $t_{x_i} - t_{\overline{x}_i} \in \cup [-1, -1] \cup [1, 1]$. We will discuss our strategy for ensuring that the values of $t_{y_j}$ and $t_{\overline{y}_j}$ differ below.

We now move on to the constraints associated with the clausal gadgets. $G_{k,l}$ represents the truth value of the $l^{th}$ element of clause $k$, and $G_{k,and}$ represents the truth value of the entire clause; each timepoint, $G_{k,*}$, that is newly created for the gadget is initialized using a contingent constraint enforcing $G - Z \in [0, 0] \cup [1, 1]$. We also create a disjunctive constraint across all gadgets, such that if for any $k$, $G_{k,and} - Z = 1$, then the constraint is satisfied. We call this disjunctive constraint the \textit{goal constraint}. This has an immediate correspondence to the notion that the entire formula $\psi$ is satisfied if any of its constituent clauses is satisfied.

\begin{figure}[!tb]
\centering
\includegraphics[width=0.75\textwidth]{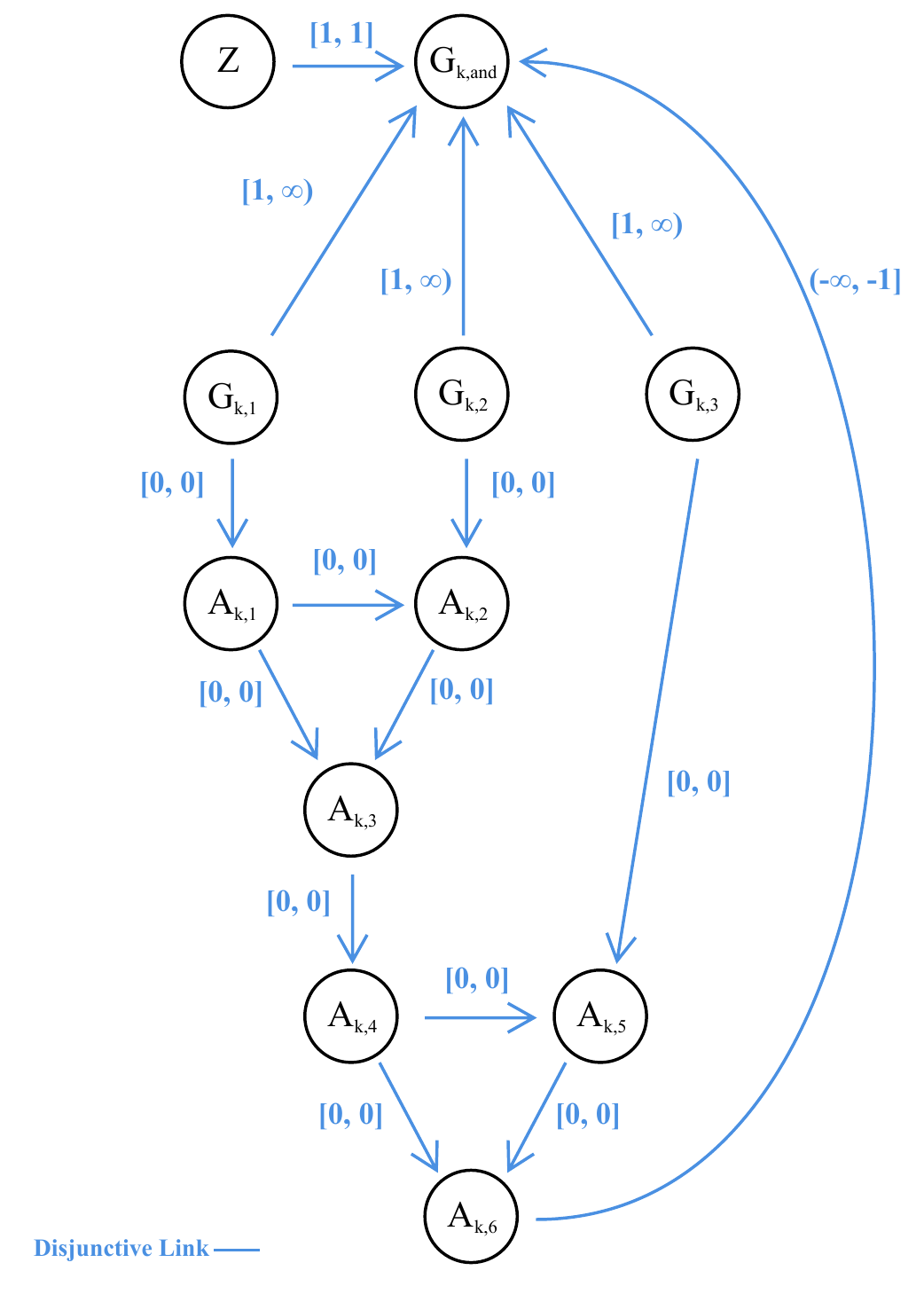}

\caption{A description of the disjunctive goal links found in each gadget used in the proof that SC-DTNU is $\csigma$-hard. The $A_k$ timepoints can each take on any value from $\{0, 1, 2\}$. The value $A_{k, 6}$ will only be precluded from taking on a value of 0 when all of $G_{k,1}, G_{k,2}, G_{k,3}$ are 1. The disjunctive constraints of this gadget are all individual parts of the larger collective disjunctive goal constraint.}
\label{fig:sc-dtnu-weak-or}
\end{figure}

\begin{figure}[!tb]
\centering
\includegraphics[width=0.75\textwidth]{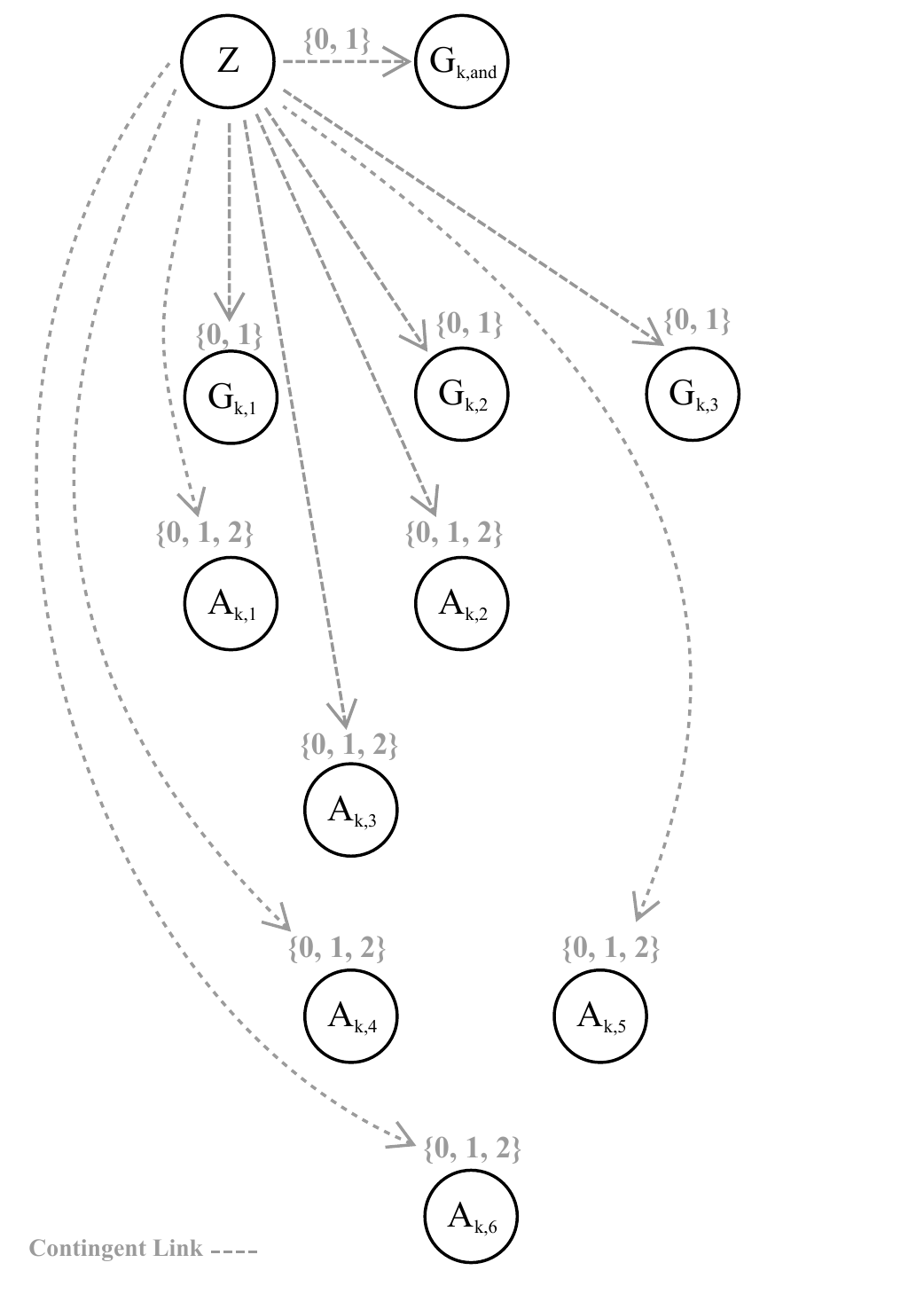}

\caption{A description of the contingent links found in each gadget used in the proof that SC-DTNU is $\csigma$-hard. The links between $Z$ and each $G_{k, l}$ are contingent links but are constrained to be equal in length to the original $x_i$, $y_j$ they relate to using the shared disjunctive goal constraint.}
\label{fig:sc-dtnu-gadget}
\end{figure}

Our current construction makes heavy use of contingent constraints, and while we may want the timepoints in our gadgets to represent certain values, their values are chosen by nature, meaning we have no way to directly control their values.

However, we do have control over the constraints of $D$ and, in particular, the disjunctive constraint that spans the gadgets. We can think of checking strong controllability as a two-player game, where the scheduler goes first and nature goes second. Nature's goal is to construct an assignment such that some constraint is violated. Upon closer examination, we see that in our construction, the only constraint that can be affected by the contingent link durations chosen by nature is the goal constraint. If there exist certain combinations of contingent link durations that we want to preclude from our evaluation, we can do so by adding additional disjunct to the goal constraint that are satisfied when those contingent links take on those durations. In this way, any contingent link values that do not conform to our desired structure make $D$ trivially controllable, and controllability then reduces to controllability under our desired set of constraints across contingent links.

First, we need to make sure that the timepoints $t_{y_j}$ and $t_{\overline{y}_j}$ take on different values. We can ensure this by adding $t_{y_j} - t_{\overline{y}_j} = 0$ to our goal constraint; if nature gives $t_{y_j}$ and $t_{\overline{y}_j}$ the same value, then we trivially ignore this case. Similarly, since we want $G_{k, l}$ to take on the same value as the literal $q$ it represents, we augment our disjunctive goal constraint with $G_{k,l} - t_q \in [-1, -1]$ and $G_{k,l} - t_q \in [1, 1]$ where $t_q$ is the timepoint associated with literal $q$. As a result, if the clausal representation of the variable differs from our assignment, our network is trivially controllable.

We enforce the conjunction of the elements of each clause by augmenting our goal constraint with $G_{k,and} - G_{k, l} \geq 1$ for each $G_{k,l}$ of our clause gadget. Since each timepoint of our gadget can take on a value of 0 or 1, this constraint will only be satisfied if some literal value is 0 while $G_{k,and}$ has a value of 1. In these situations, $G_{k,and}$ does not represent the conjunction of the literals of clause $k$, and our network then becomes trivially controllable.

Unfortunately, our network still does not perfectly encode the conjunction seen in a DNF clause. It is possible for each $G_{k, l}$ to take on a value of 1 while $G_{k, and}$ is assigned a value of 0. As a result, it may be the case that the original problem, $\exists \vc{x} \forall \vc{y} \psi(\vc{x}, \vc{y})$ is true but each $G_{k,and}$ is set to 0, meaning that the network is not strongly controllable.

To fix this, we must augment our gadget to enforce that identical inputs have the same output. This is the reason for introducing timepoints $A_{k, m}$, and these timepoints' values are set by new contingent links that enforce $A_{k,m} - Z \in [0, 0] \cup [1, 1] \cup [2, 2]$. Through an exhaustive enumeration of possible values, we can confirm that whenever $G_{k,1}, G_{k,2}, G_{k,3}$ are all 1, either $A_{k, 6}$ will be 1 or one of the disjuncts of the goal constraints (see Figure \ref{fig:sc-dtnu-weak-or}) will be satisfied. In this case, when we add $A_{k, 6} - G_{k,and} \geq 1$ to the goal constraint, we know that when $G_{k,1}, G_{k,2}, G_{k,3}$ are all equal to 1, $D$ is controllable, as either $G_{k,and} = 1$, meaning $G_{k,and} - Z = 1$, which satisfies the goal constraint, or $G_{k,and} = 0$, implying $A_{k,6} - G_{k,and} = 1$, which also satisfies the goal constraint.

Before continuing, we need to confirm that the addition of the new sub-gadget does not introduce any new problems. For all other values of $G_{k,1}$, $G_{k,2}$, and $G_{k,3}$, we know it is possible for $A_{k,6}$ to take on a value of 0. Since $A_{k,6}$ is the only timepoint of the or-gate that is related to other values by the goal constraint and setting it to 0 does not satisfy the goal constraint, we know that if $G_{k,1}, G_{k,2}, G_{k,3}$ are not all 1, then there exists a choice of values by nature such that the goal constraint is not satisfied by gadget $k$.

Our transformation is complete and because there is one gadget per clause in $\psi$ and each gadget is of constant size, we see that the transformation takes polynomial time. What remains is to show that $D$ is strongly controllable if and only if $\exists \vc{x} \forall \vc{y} \psi(\vc{x}, \vc{y})$ is true. This is evident from our construction.

If $D$ is strongly controllable, there must be some set of assignment to values $t_{x_i}$ such that no possible assignment of values to the other timepoints violates any of the constraints. We can prove this by contradiction, assuming that although our choice of $t_{x_i}$ guarantees the satisfaction of all other constraints in $D$, there is no choice of $\vc{x}$ that guarantees satisfaction of $\forall \vc{y} \psi(\vc{x}, \vc{y})$. Let $\vc{x}$ be specified such that $x_i$ is true if and only if $t_{x_i} = 1$. If $\psi$ is not guaranteed to be satisfied, there must be some $\vc{y}$ such that $\psi(\vc{x}, \vc{y})$ is false. Returning to $D$, assume that nature specifies $t_{y_j}$ such that $t_{y_j} = 1$ if and only if $y_j$ is true. Since $D$ is strongly controllable, we know that some disjunctive goal constraint is satisfied no matter the assignment of contingent timepoint variables. Let's assume that all $t_{\overline{y}_j}$ are chosen such that they represent the negation of their corresponding $t_{y_j}$, that all $A_{k, m}$ of the gadgets are chosen such that the disjunctive constraints involved between all $G_{k,l}$ and $A_{k,m}$ are not satisfied, and that all $G_{k, and}$ are chosen to be 0. The only remaining disjunctive constraints are those involving each $G_{k, and}$. For any particular $k$, setting $G_{k, and}$ to 0 only satisfies a constraint if $A_{k, 6}$ is 1, so given all these assumptions, at least one $A_{k, 6}$ must be set to 1 (otherwise the system would be uncontrollable). As we demonstrated earlier, $A_{k, 6}$ is only constrained to be 1 when all of $G_{k, 1}$, $G_{k, 2}$, and $G_{k, 3}$ are also 1. But those three values correspond exactly to literals in a clause of $\psi(\vc{x}, \vc{y})$. If all three are 1, then we have a true clause and because $\psi$ is a 3-DNF formula, this means that $\psi$ is true. We have a contradiction. Therefore if $D$ is controllable, $\exists \vc{x} \forall \vc{y} \psi(\vc{x}, \vc{y})$.

To conclude we show the reverse direction, that if $\exists \vc{x} \forall \vc{y} \psi(\vc{x}, \vc{y})$ is true, then $D$ is strongly controllable. Let $\vc{x}$ be the assignment of variables that guarantees $\forall \vc{y} \psi(\vc{x}, \vc{y})$; we show how we can use $\vc{x}$ to show that $D$ is strongly controllable. We will pick our $t_{x_i}$ such that $t_{x_i} = 1$ if and only if $x_i$ is true and will pick our $t_{\overline{x}_i}$ such that $t_{\overline{x}_i} \neq t_{x_i}$. Again we will proceed with proof by contradiction, assuming that $D$ is not strongly controllable. Our choice of $t_{x_i}$ and $t_{\overline{x}_i}$ satisfy all constraints except the disjunctive goal constraint, so there must be a choice of contingent timepoints that violate the disjunctive goal constraint. We know setting $G_{k, and} = 1$ satisfies the goal constraint, so all $G_{k, and} = 0$. By proxy, for all $k$, $A_{k, 6} = 0$ to ensure that the goal constraint is not satisfied because of the link between $A_{k, 6}$ and $G_{k, and}$. Because $A_{k, 6} = 0$, it must be the case that for each gadget, at least one of $G_{k, 1}$, $G_{k, 2}$, or $G_{k, 3}$ must equal zero. In order for the goal disjunctive constraint to remain unsatisfied, each $G_{k, l}$ must maintain the same value as some $t_{x_i}$, $t_{\overline{x}_i}$, $t_{y_j}$, or $t_{\overline{y}_j}$ based on the values of the clauses of $\phi$. This forces a particular assignment of values to $t_{y_j}$ which we can map back to some $\vc{y}$. For that particular assignment, we know that $\psi(\vc{x}, \vc{y})$ is true, or that there is some clause $k'$ with all literals set to true. This contradicts the fact that for all $k$, at least one of $G_{k, 1}$, $G_{k, 2}$, or $G_{k, 3}$ must be zero. Thus, $D$ must be strongly controllable, and we have proven that SC-DTNU is $\csigma$-hard.

\end{proof}

\subsection{Completeness}

We now move on to proving completeness for the controllability problems on each temporal network. Our general approach for characterizing the complexity of controllability problems will be to map an inputted temporal network to a corresponding system of conditional linear inequalities that encode the same constraints. We will then use existential and universal quantifiers over the variables to dictate which type of controllability is being determined.

Our transformation proceeds as follows. We can imagine the execution of a temporal network as being a game played between two agents, the scheduler and nature, where the scheduler assigns times to executable timepoints and nature assigns times to contingent timepoints. In general the question of determining controllability will reduce to the problem of evaluating a quantified linear system and our techniques draw inspiration from and are related to approaches in those areas \cite{eirinakis:qli,subramani:qlp}.

For notational convenience, we will split our variables into $\vc{x}$ and $\vc{y}$ for those assigned by the scheduler and nature, respectively. For each executable timepoint $e_i$, we create a new variable $x_i$, and for each contingent timepoint $c_i$, we create a new variable $y_i$.

We create a one-to-one mapping between the set of temporal network constraints and the new linear inequalities. First, we replace all executable timepoints $e_i$ with the corresponding $x_i$. With the contingent timepoints, however, we need to be more careful. For each contingent timepoint $c_i$, we find the contingent constraint that restricts it of the form $l_c \leq c_i - e_j \leq u_c$. We then replace each instance of $c_i$ in our constraints with $y_i + x_j$. Our reason for doing this has to do with the nature of contingent constraints. In temporal networks, there is a guarantee that nature respects the contingent constraint bounds in relation to its corresponding starting executable timepoint. So while free constraints relate timepoints in terms of the absolute time of their occurrence, contingent constraints require nature to respect relative timings of events. If the durations of contingent constraints are to be known before scheduling begins, as is the case with weak controllability, then the constructed system of linear inequalities will fail to map to the base temporal network if nature is asked to pick the precise times of contingent events.

After the substitutions, each constraint is a combination of conditional linear inequalities of the form $\psi \rightarrow \vc{a} \cdot \begin{bmatrix} \vc{x} \\ \vc{y} \end{bmatrix} \leq b$, where $b$ is some constant, $\psi$ is a (possibly empty) precondition for the enforcement of the constraint, and $\vc{a}$ represents the coefficients of the constraints where each coefficient is either -1, 0, or 1. Since all constraints are relative, without loss of generality, we can say that the earliest event happens at time $t = 0$, meaning we can safely require that $\vc{x} \geq 0$. When we quantify over variables to pick controllability, we require that each $x_i$ has an existential quantifier and each $y_i$ has a for-all quantifier drawn from the union of the ranges $[l_1, u_1], ..., [l_d, u_d]$, where $l_j$ and $u_j$ are retrieved from one of $c_i$'s corresponding contingent constraints.

When evaluating controllability for disjunctive networks, it is useful to consider each contingent range separately, and so we will define $\Omega$ as a mapping from each variable $y_i$ and one of its possible continuous ranges. In general, we will use the shorthand $\forall \Omega$ to indicate that we are considering all possible mappings and $\forall \vc{y} \in \Omega$ to specify that we are drawing our $\vc{y}$ from one particular mapping. Our choice of the ordering of the quantifiers will dictate which type of controllability will be considered. We also must consider how conditions affect our model, and will define $\Psi$ as the full set of conditions that can be observed by the scheduler when our temporal networks include conditional constraints.

It is also worth noting that whenever we consider a vector of values $\vc{x_c}$ that represent a solution to our scheduling problem, we assume that the representation of $\vc{x_c}$ is polynomial in the size of the original input. While we are agnostic to which particular representation is used, we do still require a fixed number of bits required to represent each individual number. The implication of this is that between any two numbers, there are a finite number of intermediate values that can be represented.

The rest of our analysis is divided into an analysis of strong controllability over temporal networks, weak controllability over temporal networks, and dynamic controllability over temporal networks.

\subsubsection{Strong Controllability}

\begin{theorem}
\label{thm:cstnu-sc-p}
Checking strong controllability of a CSTNU is in P.
\end{theorem}

\begin{proof}
We start by encoding the SC-CSTNU problem in our described format:
$$\exists \vc{x} \forall \vc{y} \forall \Psi : \bigwedge\limits_i \psi_i \rightarrow \vc{a_i} \cdot \begin{bmatrix}\vc{x} \\ \vc{y} \end{bmatrix} \leq b_i$$
Because $\forall \bigwedge \phi$ is the same as $\bigwedge \forall \phi$, we can rewrite our problem as:
$$\exists \vc{x} \forall \vc{y} \bigwedge\limits_i \forall \Psi : \psi_i \rightarrow \vc{a_i} \cdot \begin{bmatrix}\vc{x} \\ \vc{y} \end{bmatrix} \leq b_i$$
Since the inner equation must hold for all $\Psi$, it must also hold when $\psi_i$ is true, allowing us to eliminate the conditionals:
$$\exists \vc{x} \forall \vc{y} \bigwedge\limits_i : \vc{a_i} \cdot \begin{bmatrix}\vc{x} \\ \vc{y} \end{bmatrix} \leq b_i$$
But of course, this is exactly the encoding for checking strong controllability of an STNU. Since STNU strong controllability is verifiable in polynomial time \cite{vidal:controllability}, our work demonstrates that strong controllability of a CSTNU can be determined in polynomial time through reduction to an STNU.
\end{proof}

\begin{theorem}
Checking the strong controllability of TCSPUs is NP-complete.
\end{theorem}

\begin{proof}
We know that checking the feasibility of a TCSP is NP-hard \cite{dechter:temporal}, and because TCSPUs are a generalization of TCSPs, it follows that SC-TCSPU is NP-hard. To prove completeness, we show that SC-TCSPU $\in$ NP.

In a TCSPU, all disjunctive requirement links span the same pair of variables, meaning that every requirement link is of the form $t_i - t_j \in [l_1, u_1] \cup ... \cup [l_k, u_k]$, where for every $p < q$, $u_p < l_q$. This allows us to rewrite all individual constraints as:
$$t_i - t_j \geq l_1 \wedge \left(\bigwedge\limits_{p=2}^k t_i - t_j \leq u_{p-1} \vee t_i - t_j \geq l_p \right) \wedge$$
$$t_i - t_j \leq u_k$$
Now when we encode strong controllability of a TCSPU, we can write the formula as:
$$\exists \vc{x} \forall \vc{y} : \bigwedge\limits_i \bigvee\limits_j \vc{a_{ij}} \cdot \begin{bmatrix}\vc{x} \\ \vc{y} \end{bmatrix} \leq b_{ij}$$
$$\exists \vc{x} \bigwedge\limits_i \forall \vc{y} : \bigvee\limits_j \vc{a_{ij}} \cdot \begin{bmatrix}\vc{x} \\ \vc{y} \end{bmatrix} \leq b_{ij}$$
$$\exists \vc{x} \bigwedge\limits_i \neg \exists \vc{y} : \bigwedge\limits_j \vc{a_{ij}} \cdot \begin{bmatrix}\vc{x} \\ \vc{y} \end{bmatrix} > b_{ij}$$

For any fixed $\hat{x}$ and $i$, we can solve the problem $\exists \vc{y} : \bigwedge\limits_j [\hat{x}^T; \vc{y}^T] \cdot \vc{a_{ij}} > b_{ij}$ in polynomial time. We know that linear programs can be optimized in polynomial time \cite{karmarkar:lp}, and so to derive an answer for our original problem, we solve the linear program $\bigwedge\limits_j [\vc{a_{ij}}^T; -1] \begin{bmatrix} \hat{x} \\ \vc{y} \\ \epsilon \end{bmatrix} \geq b_{ij}$ maximizing $\epsilon$. If no solution exists, then there is no valid $\vc{y}$. If a solution exists with $\epsilon \leq 0$, then there was some constraint for which $[\hat{x}^T; \vc{y}^T] \cdot \vc{a_{ij}} > b_{ij}$ did not hold as there was a non-positive margin required to make all inequalities hold. Thus, only if $\epsilon > 0$, do we say that there exists a $\vc{y}$ satisfying our original constraints.

This immediately implies that we have a routine for verifying a certificate for SC-TCSPU in polynomial time. Given a certificate $\hat{x}$, then for each of the constraints $i$, we run our subroutine for determining whether a $\vc{y}$ exists that satisfies the specified sub-constraints. Since the verification algorithm runs in polynomial time, we know that SC-TCSPU $\in$ NP, and that SC-TCSPU is NP-complete.
\end{proof}

\begin{theorem}
Checking the strong controllability of DTNUs and CDTNUs are $\csigma$-complete.
\end{theorem}

\begin{proof}
We know that checking the strong controllability of a DTNU is $\csigma$-hard from Lemma \ref{sc-dtnu-hard} and because CDTNUs generalize DTNUs, SC-CDTNU is also $\csigma$-hard. To demonstrate that both are $\csigma$-complete, we show that checking the strong controllability of a CDTNU is in $\csigma$.

To do so, we show that with an NP oracle we can verify that a CDTNU is strongly controllable in polynomial time. We start with an encoding of our problem:
$$\exists \vc{x} \forall \Psi \forall \Omega \forall \vc{y} \in \Omega :
\bigwedge\limits_i \bigvee\limits_j \psi_{ij} \rightarrow \bigwedge\limits_k \vc{a}_{ijk} \cdot \begin{bmatrix} \vc{x} \\ \vc{y} \end{bmatrix} \leq b_{ijk}$$
and we let our certificate be the assignment of values to all executable timepoints, $\hat{x}$. Given this certificate, an NP-oracle is capable of evaluating:
$$\exists \Psi \exists \Omega \exists \vc{y} \in \Omega : \neg \bigwedge\limits_i \bigvee\limits_j \psi_{ij} \rightarrow \bigwedge\limits_k \vc{a}_{ijk} \cdot \begin{bmatrix} \hat{x} \\ \vc{y} \end{bmatrix} \leq b_{ijk}$$
We can see this simply, as when we provide a certificate comprised of $\hat{\Psi}, \hat{\Omega}, \hat{y}$, it takes linear time to verify whether the conditional constraints are all satisfied.

Thus, when given a candidate assignment $\hat{x}$, we can use an NP-oracle to evaluate the negation of the remainder of the formula. If the negation has no solution, then we know that the original formula is true, and we have a way to verify SC-CDTNU in polynomial time. Thus, SC-CDTNU $\in \csigma$, so SC-DTNU and SC-CDTNU are $\csigma$-complete.
\end{proof}

\subsubsection{Weak Controllability}

Next, we move on to evaluating the complexity of weak controllability in temporal networks.

\begin{theorem}
Checking the weak controllability of CSTNUs is coNP-complete.
\end{theorem}

\begin{proof}
Checking the weak controllability of STNUs is coNP-complete \cite{morris:waypoint}, so similarly checking the weak controllability of CSTNUs must be coNP-hard. To demonstrate that WC-CSTNU is coNP-complete, we must show that WC-CSTNU $\in$ coNP. We see this clearly when we look at the quantified linear system we get when evaluating a CSTNU's weak controllability:
$$\forall \Psi \forall \vc{y} \exists \vc{x} : \bigwedge\limits_i \psi_i \rightarrow \vc{a}_i \cdot \begin{bmatrix} \vc{x} \\ \vc{y} \end{bmatrix} \leq b_i$$

To show that WC-CSTNU is in coNP, we show that its complement problem is in NP, or that we can verify the following formula in polynomial time:
$$\exists \Psi \exists \vc{y} \neg \exists \vc{x} : \bigwedge\limits_i \psi_i \rightarrow \vc{a}_i \cdot \begin{bmatrix} \vc{x} \\ \vc{y} \end{bmatrix} \leq b_i$$
In this instance, we take as our certificate a particular choice of $\hat{\Psi}$ and $\hat{y}$. We can verify these values directly:
$$\neg \exists \vc{x} : \bigwedge\limits_{i : \hat{\Psi} \vDash \psi_i} \psi_i \rightarrow \vc{a}_i \cdot \begin{bmatrix} \vc{x} \\ \hat{y} \end{bmatrix} \leq b_i$$
$$\neg \exists \vc{x} : \bigwedge\limits_{i : \hat{\Psi} \vDash \psi_i} \vc{a}_i \cdot \begin{bmatrix} \vc{x} \\ \hat{y} \end{bmatrix} \leq b_i$$
Of course, we can evaluate all linear inequalities simultaneously through the evaluation of a linear program:
$$\neg \exists \vc{x} : A_{\hat{\Psi}}\begin{bmatrix} \vc{x} \\ \hat{y} \end{bmatrix} \leq \vc{b}_{\hat{\Psi}}$$
Since linear programs can be evaluated in polynomial time \cite{karmarkar:lp}, we can verify the complement of WC-CSTNU in polynomial time, meaning that WC-CSTNU $\in$ coNP and is coNP-complete.
\end{proof}

\begin{theorem}
Checking the weak controllability of TCSPUs, DTNUs, and CDTNUs are $\cpi$-complete.
\end{theorem}

\begin{proof}

By Lemma \ref{wc-tcspu-hard}, we know that computing the weak controllability of TCSPUs are $\cpi$-hard, meaning computing WC-DTNU and WC-CDTNU are both also $\cpi$-hard. To show that all three are $\cpi$-complete, we must show that WC-CDTNU $\in \cpi$. We start with the quantified formula representation of weak controllability in a CDTNU:
$$\forall \Psi \forall \Omega \forall \vc{y} \in \Omega \exists \vc{x} : \bigwedge\limits_i \bigvee\limits_j \psi_{ij} \rightarrow \left( \bigwedge\limits_k \vc{a}_{ijk} \cdot \begin{bmatrix} \vc{x} \\ \vc{y} \end{bmatrix} \leq b_{ijk} \right)$$
For our purposes, it will be useful to show that the complementary problem is in $\csigma$:
$$\exists \Psi \exists \Omega \exists \vc{y} \neg \exists \vc{x} : \bigwedge\limits_i \bigvee\limits_j \psi_{ij} \rightarrow \left( \bigwedge\limits_k \vc{a}_{ijk} \cdot \begin{bmatrix} \vc{x} \\ \vc{y} \end{bmatrix} \leq b_{ijk} \right)$$
To prove that solving the above formula is in $\csigma$, we show that with an NP-oracle, we can construct a verification algorithm that runs in polynomial time. Our verifier will take in the certificate composed of $\hat{\Psi}, \hat{\Omega}, \hat{y}$, leaving the subproblem:
$$\neg \exists \vc{x} : \bigwedge\limits_i \bigvee\limits_{j: \hat{\Psi} \vDash \psi_{ij}} \psi_{ij} \rightarrow \left( \bigwedge\limits_k \vc{a}_{ijk} \cdot \begin{bmatrix} \vc{x} \\ \hat{y} \end{bmatrix} \leq b_{ijk} \right)$$
$$\neg \exists \vc{x} : \bigwedge\limits_i \bigvee\limits_{j: \hat{\Psi} \vDash \psi_{ij}} \bigwedge\limits_k \vc{a}_{ijk} \cdot \begin{bmatrix} \vc{x} \\ \hat{y} \end{bmatrix} \leq b_{ijk}$$
The unnegated version of this problem is clearly in NP. Given a certificate $\hat{x}$, we can verify whether or not the formula holds in linear time. As a result, with an NP-oracle, we can solve the presented subproblem, meaning that our complement problem is in $\csigma$ and our original problem is thus in $\cpi$. This proves that WC-TCSPU, WC-DTNU, and WC-CDTNU are $\cpi$-complete.
\end{proof}

\subsubsection{Dynamic Controllability}

Finally, we show that checking dynamic controllability for any temporal network with uncertainty and either disjunctions or conditional constraints is PSPACE-complete.

\begin{theorem}
Checking the dynamic controllability of CSTNUs, TCSPUs, DTNUs, and CDTNUs are PSPACE-complete.
\end{theorem}

\begin{proof}
We know from Lemma \ref{dc-tcspu-hard} that DC-TCSPU is PSPACE-hard, meaning that checking the dynamic controllability of DTNUs and CDTNUs must also be PSPACE-hard. Similarly because checking the dynamic controllability of CSTNs is PSPACE-hard \cite{cairo:cstn-pspace-complete}, DC-CSTNU must also be PSPACE-hard. In order to show that determining dynamic controllability for any of these four networks in PSPACE-complete, we provide an algorithm for checking the dynamic controllability of CDTNUs which requires polynomial space (see Algorithm \ref{alg:cdtnu-dc-pspace}).

\begin{algorithm}[!t]
\SetAlgoLined
\SetKwInput{Input}{Input}
\SetKwInput{Output}{Output}
\SetKwInput{Algorithm}{\textsc{CheckDC}}
\SetKwInput{Initialize}{Initialization}
\SetKwIF{If}{ElseIf}{Else}{if}{then}{else if}{else}{endif}
\Indm
\Input{A list of timepoints with assigned values, $T$;\\A list of active contingent links, $A$;\\A set of yet-to-be-executed timepoints $E$;\\The input CDTNU $G$;\\The current time, $\tau$}
\Output{Whether the CDTNU is dynamically controllable.}
\Algorithm{}
\Indp
\If{$E.empty()$} {
    \For{$realization \in A.realizationsFrom(\tau)$} {
        \If{$!G.isConsistent(T.extend(realization))$} {
            \Return{$false$}\;
        }
    }
    \Return{$true$}\;
}

\For{$t \in E$} {
    \For{$\tau' \in [\tau, G.tMax]$} {
        $allSatisfied \leftarrow true$\;
        \For{$realization \in A.realizationsFrom(\tau)$} {
            $earliest \leftarrow realization.earliest()$\;
            \If{$early.time \leq \tau'$} {
                \If{!\textsc{CheckDC}($T \cup \{earliest\},$\\$A.nextContingents(earliest),$\\$E,$\\$G,$\\$earliest.time$)} {
                    $allSatisfied \leftarrow false$\;
                    break\;
                }
            } \Else {
                \If{!\textsc{CheckDC}($T \cup \{\textsc{TimepointAssignment}(t, \tau')\},$\\$A.nextContingents(\textsc{TimepointAssignment}(t, \tau')),$\\$E \setminus t,$\\$G,$\\$\tau'$)} {
                    $allSatisfied \leftarrow false$\;
                    break\;
                }
            }
        }
        \If{$allSatisfied$} {
            \Return $true$\;
        }
    }
}
\Return{$false$}\;
\caption{PSPACE algorithm for checking DC-CDTNU.}
\label{alg:cdtnu-dc-pspace}
\end{algorithm}

Before we explain the details of the algorithm, we need to extend some concepts to describe a partially executed CDTNU, as our algorithm for determining dynamic controllability will recursively act on partially executed networks. We say that timepoints are \textit{scheduled} if they have been assigned a specific value, whether by the scheduler or by nature. We say that a contingent link is \textit{active} if its starting timepoint has been scheduled but its ending timepoint has not. Finally, given a group of active contingent links, we say that the set of all \textit{realizations} from some time $\tau$ is the set of all possible times at which the contingent links could end with none of them ending sooner than $\tau$. We can now move on to describing the function of the algorithm before demonstrating that it uses at most polynomial space.

The algorithm works by recursively simulating all possible strategies used by an agent in a dynamically controllable setting. As input, it takes in a list of timepoints whose values have already been scheduled (either by the scheduler or by nature), a list of active contingent links, a list of unexecuted timepoints, the CDTNU, and the current time. While there are still executable timepoints that need to be scheduled, the algorithm searches for one that guarantees a valid dynamically controllable strategy.

In the context of dynamic controllability, an agent only has one of two possible actions: they can either unconditionally schedule an action or schedule an action to occur so long as no other contingent timepoint occurs in the interim. We model this behavior by modeling all scheduling actions as interruptible by contingent timepoints. In other words, if a contingent timepoints occurs before an event we unconditionally scheduled, we still give the agent the choice to adapt their strategy. In the case of an unconditionally scheduled action, they would just reaffirm their previous choice.

To model all strategies, we iterate over all possible timepoints that can be scheduled (line 6) and all possible times at which they can be scheduled (line 7). If at least one scheduling of a timepoint given the input parameters is valid, then we know that the CDTNU is dynamically controllable. When there are no active contingent links that might be scheduled before the timepoint that we chose to schedule, we can recurse on that assignment to get an answer (line 20-26), but in the case that there are contingent links that may occur earlier, we have to respond to them in turn (lines 9-18). If all possible realizations of contingent link values still guarantee that the CDTNU is consistent, then we know that the system is dynamically controllable.

Now, we show that the algorithm uses at most polynomial space. If we have no more timepoints to schedule, then we stay in lines 1-5 of the algorithm, which checks consistency over all possible realizations of the remaining contingent links. Checking consistency is a polynomial time operation, as it just requires iterating through each constraint and verifying that it is satisfied. While we have to do this for exponentially many realizations, we do not need to remember each particular realization; we merely need to remember the current realization and know how to increment to the next one. As a result, handling all realizations also takes polynomial space.

In the event that we do have executable timepoints to schedule, we operate over lines 6-29. Iterating over each timepoint at line 6 takes polynomial space, and when we iterate over all $\tau'$ at line 7, we have a finite but exponentially large number of values to choose from but only need polynomially many bits to represent that choice. At line 9, we handle realizations in the same way as we did at line 2, meaning we only need polynomial space, and then we have the remaining recursive calls. If we look at the number of possible stack frames, we see that every time we recursively call \textsc{CheckDC}, we add a new timepoint to $T$, correspondingly removing it from either $A$ (line 13) or $E$ (line 22). The set $E$ never grows, and because every contingent link's ending timepoint is unique, we will never add the same timepoint to $A$ twice. This means that after at most $|X_e \cup X_b|$ recursive calls, we will eventually reach a state where $E$ is empty, and our recursive calls will terminate. This preserves our guarantee that we use polynomial space, meaning that Algorithm \ref{alg:cdtnu-dc-pspace} decides DC-CDTNU and is in PSPACE. Thus, deciding dynamic controllability for CSTNUs, TCSPUs, DTNUs, and CDTNUs are all PSPACE-complete.
\end{proof}

\section{Discussion and Conclusions}

Our work provides novel complexity results that are much tighter than existing bounds and require at most polynomial space for strong, weak, and dynamic controllability of several distinct types of temporal networks; this work is summarized in Figure \ref{fig:taxonomy}. Beyond the contribution of the relevant proofs, the value of these results is that it gives modelers insight into which types of features have a significant impact on the runtime complexity of a problem. Many of these insights are not immediately obvious, and in the remainder of this section we discuss a few of them.

First we consider CSTNUs. CSTNUs are a generalization of CSTNs and STNUs and share much in common with their predecessors. In particular, strong controllability of CSTNUs, being in P, can be computed quite efficiently. Our proof for Theorem \ref{thm:cstnu-sc-p} actually proves a stronger result that a CSTNU is strongly controllable if and only if the corresponding STNU derived by making all constraints unconditional is strongly controllable. This implies that strong controllability of CSTNUs can be computed in $O(mn)$ time, which is as fast as it takes to compute the feasibility of a simple STN. When we turn to weak and dynamic controllability, we see that checking the controllability of a CSTNU is in the same class as checking controllability of a CSTN. From the perspective of the modeler, this implies that there is a surprisingly low cost to adding uncertainty to a temporal constraint model that already uses conditional constraints.

While CSTNU controllability checking matches the complexity of CSTN controllability checking, it only matches the controllability checking complexity of strong and weak controllability for STNUs. In fact, dynamic controllability checking across all types of networks, with the exception of STNUs, is PSPACE-complete. In scheduling problems, modelers must often make the trade-off between using strong controllability, which is often easier to compute, and dynamic controllability, which gives more flexibility during execution but is more expensive. In instances where dynamic controllability is deemed necessary, there is a significant advantage to relaxing the underlying temporal model, eliminating conditional and disjunctive constraints, to use an STNU. It is still quite surprising that despite the fact that STNU dynamic controllability can be determined in polynomial time, every other modification makes computing dynamic controllability PSPACE-complete.

A final area worth discussing is the effect of temporal disjunctions. The two temporal network models that use disjunctions without temporal uncertainty are TCSPs and DTNs; TCSPs have simple disjunctions, only requiring disjunctions over a single link, while DTNs have full disjunctions, allowing disjunctions to span multiple links. Since determining feasibility for both network structures is NP-complete, intuition would suggest that after adding uncertainty the complexity of checking controllability for TCSPUs and DTNUs would also be the same. While this is the case for weak and dynamic controllability, we do see a difference in strong controllability, meaning that strong controllability is easier to compute in TCSPUs than it is in DTNUs, assuming NP $\neq \csigma$, implying there is a meaningful difference between the two types of disjunctions.

As we look forward, there are still many areas worthy of future research efforts. One, in particular, is the development of novel algorithms for determining the controllability of these networks. Our work establishes bounds on the complexity of computing controllability but does minimal work to provide algorithms for doing so. In practice, our proofs admit the trivial polynomial-space strategy of recursive enumeration of certificates but these algorithms are likely impractical. Our new theoretical bounds open up the challenge of finding novel algorithms that are reasonable for practical use while still respecting polynomial time bounds.

\section*{Acknowledgments}

\noindent This research was funded in part by the Toyota Research Institute under grant number LP-C000765-SR.

\section*{Bibliography}

\bibliography{aaai.bib}
\bibliographystyle{aaai}

\end{document}